\documentclass[12pt]{article}

\title{Stochastic Process Bandits: Upper Confidence Bounds Algorithms via Generic Chaining}
\usepackage{fullpage}
\usepackage{microtype}
\usepackage{natbib,authblk}
\usepackage{custommacros}
\author{Emile Contal}
\author{Nicolas Vayatis}
\affil{CMLA, ENS Cachan, CNRS, Universit\'e Paris-Saclay, 94235 Cachan, France}
\date{}

\begin{document}

\maketitle

\begin{abstract}
The paper considers the problem of global optimization in the setup of stochastic process bandits.
We introduce an UCB algorithm which builds a cascade of discretization trees based on generic chaining
in order to render possible his operability over a continuous domain.
The theoretical framework applies to functions
under weak probabilistic smoothness assumptions
and also extends significantly the spectrum of application of UCB strategies.
Moreover generic regret bounds are derived which are then specialized to Gaussian processes
indexed on infinite-dimensional spaces as well as to quadratic forms of Gaussian processes.
Lower bounds are also proved in the case of Gaussian processes to assess the optimality of the proposed algorithm.
\end{abstract}

\section{Introduction}
Among the most promising approaches to address the issue of global optimization
of an unknown function under reasonable smoothness assumptions
comes from extensions of the multi-armed bandit setup.
\cite{Bubeck2009} highlighted the connection between cumulative regret
and simple regret which facilitates fair comparison between methods
and \cite{Bubeck2011} proposed  bandit algorithms on metric space $\cX$, called $\cX$-armed bandits.
In this context, theory and algorithms have been developed
in the case where the expected reward is a function $f:\cX\to\bR$
which satisfies certain smoothness conditions such as Lipschitz or H\"older continuity
\citep{Kleinberg2004, Kocsis2006, Auer2007, Kleinberg2008, Munos2011}.
Another line of work is the Bayesian optimization framework \citep{Jones1998, Bull2011, Mockus2012}
for which the unknown function $f$ is assumed to be the realization of a prior stochastic process distribution,
typically a Gaussian process.
An efficient algorithm that can be derived in this framework is the popular GP-UCB algorithm
due to \cite{Srinivas2012}.
However an important limitation of the upper confidence bound (UCB) strategies
without smoothness condition
is that the search space has to be {\em finite} with bounded cardinality,
a fact which is well known but, up to our knowledge,
has not been discussed so far in the related literature.

In this paper, we propose an approach which improves both lines of work with respect to their present limitations.
Our purpose is to: (i) relax smoothness assumptions that limit the relevance
of $\cX$-armed bandits in practical situations where target functions may only display random smoothness,
(ii) extend the UCB strategy for arbitrary sets $\cX$.
Here we will assume that $f$, being the realization of a given stochastic process distribution,
fulfills a \emph{probabilistic smoothness} condition.
We will consider the stochastic process bandit setup and we develop a UCB algorithm
based on {\em generic chaining} \citep{Bogachev1998,Adler2009,Talagrand2014,Gine2015}.
Using the generic chaining construction,
we compute hierarchical discretizations of $\cX$ under the form of chaining trees
in a way that permits to control precisely the discretization error.
The UCB algorithm then applies on these successive discrete subspaces
and chooses the accuracy of the discretization at each iteration
so that the cumulative regret it incurs matches the state-of-the art bounds on finite $\cX$.
In the paper, we propose an algorithm which computes a generic chaining tree
for arbitrary stochastic process in quadratic time.
We show that this tree is optimal for classes like Gaussian processes with high probability. 
Our theoretical contributions have an impact in the two contexts mentioned above.
From the bandit and global optimization point of view,
we provide a generic algorithm that incurs state-of-the-art regret on stochastic process objectives
including non-trivial functionals of Gaussian processes such as
the sum of squares of Gaussian processes (in the spirit of mean-square-error minimization), 
or nonparametric Gaussian processes on ellipsoids (RKHS classes),
or the Ornstein-Uhlenbeck process, which was conjectured impossible by \citep{Srinivas2010} and \citep{Srinivas2012}.
From the point of view of Gaussian process theory,
the generic chaining algorithm leads to tight bounds on the supremum of the process
in probability and not only in expectation.

The remainder of the paper is organized as follows.
In Section~\ref{sec:framework}, we present the stochastic process bandit framework over continuous spaces.
Section~\ref{sec:chaining} is devoted to the construction of generic chaining trees for search space discretization.
Regret bounds are derived in Section~\ref{sec:regret} after choosing adequate discretization depth.
Finally, lower bounds are established in Section~\ref{sec:lower_bound}.

\section{Stochastic Process Bandits Framework}
\label{sec:framework}

We consider the optimization of an unknown function $f:\cX\to\bR$
which is assumed to be sampled from a given separable stochastic process distribution.
The input space $\cX$ is an arbitrary space not restricted to subsets of $\bR^D$,
and we will see in the next section how the geometry of $\cX$ for a particular metric
is related to the hardness of the optimization.
An algorithm iterates the following:
\begin{itemize}
\item it queries $f$ at a point $x_i$ chosen with the previously acquired information,
\item it receives a noisy observation $y_i=f(x_i)+\epsilon_t$,
\end{itemize}
where the $(\epsilon_i)_{1\le i \le t}$ are independent centered Gaussian $\cN(0,\eta^2)$ of known variance.
We evaluate the performances of such an algorithm using $R_t$ the cumulative regret:
\[R_t = t\sup_{x\in\cX}f(x) - \sum_{i=1}^t f(x_i)\,.\]
This objective is not observable in practice,
and our aim is to give theoretical upper bounds that hold with arbitrary high probability
in the form:
\[\Pr\big[R_t \leq g(t,u)\big] \geq 1-e^{-u}\,.\]
Since the stochastic process is separable, the supremum over $\cX$ can be replaced by
the supremum over all finite subsets of $\cX$ \citep{Boucheron2013}.
Therefore we can assume without loss of generality that $\cX$ is finite with arbitrary cardinality.
We discuss on practical approaches to handle continuous space in Appendix~\ref{sec:greedy_cover}.
Note that the probabilities are taken under the product space of both the stochastic process $f$ itself
and the independent Gaussian noises $(\epsilon_i)_{1\le i\le t}$.
The algorithm faces the exploration-exploitation tradeoff.
It has to decide between reducing the uncertainty on $f$
and maximizing the rewards.
In some applications one may be interested in finding the maximum of $f$ only,
that is minimizing $S_t$ the simple regret:
\[S_t = \sup_{x\in\cX}f(x) - \max_{i\leq t}f(x_i)\,.\]
We will reduce our analysis to this case by simply observing that $S_T\leq \frac{R_T}{T}$.

\paragraph{Confidence Bound Algorithms and Discretization.}
To deal with the uncertainty,
we adopt the \emph{optimistic optimization} paradigm
and compute high confidence intervals where the values $f(x)$ lie with high probability,
and then query the point maximizing the upper confidence bound \citep{Auer2002}.
A naive approach would use a union bound over all $\cX$
to get the high confidence intervals at every points $x\in\cX$.
This would work for a search space with fixed cardinality $\abs{\cX}$,
resulting in a factor $\sqrt{\log\abs{\cX}}$ in the Gaussian case,
but this fails when $\abs{\cX}$ is unbounded,
typically a grid of high density approximating a continuous space.
In the next section,
we tackle this challenge by employing {\em generic chaining} to build hierarchical discretizations of $\cX$.

\section{Discretizing the Search Space via Generic Chaining}
\label{sec:chaining}

\subsection{The Stochastic Smoothness of the Process}
Let $\ell_u(x,y)$ for $x,y\in\cX$ and $u\geq 0$ be the following confidence bound on the increments of $f$:
\[\ell_u(x,y) = \inf\Big\{s\in\bR: \Pr[f(x)-f(y) > s] < e^{-u}\Big\}\,.\]
In short, $\ell_u(x,y)$ is the best bound satisfying $\Pr\big[f(x)-f(y) \geq \ell_u(x,y)\big] < e^{-u}$.
For particular distributions of $f$, it is possible to obtain closed formulae for $\ell_u$.
However, in the present work we will consider upper bounds on $\ell_u$.
Typically, if $f$ is distributed as a centered Gaussian process of covariance $k$,
which we denote $f\sim\cGP(0,k)$, we know that $\ell_u(x,y) \leq \sqrt{2u}d(x,y)$,
where $d(x,y)=\big(\E(f(x)-f(y))^2\big)^{\frac 1 2}$ is the canonical pseudo-metric of the process.
More generally, if it exists a pseudo-metric $d(\cdot,\cdot)$ and a function $\psi(\cdot,\cdot)$
bounding the logarithm of the moment-generating function of the increments, that is,
\[\log \E e^{\lambda(f(x)-f(y))} \leq \psi(\lambda,d(x,y))\,,\]
for $x,y\in\cX$ and $\lambda\in I \subseteq \bR$,
then using the Chernoff bounding method \citep{Boucheron2013},
\[\ell_u(x,y) \leq \psi^{*-1}(u,d(x,y))\,,\]
where $\psi^*(s,\delta)=\sup_{\lambda\in I}\big\{\lambda s - \psi(\lambda,\delta)\big\}$
is the Fenchel-Legendre dual of $\psi$
and $\psi^{*-1}(u,\delta)=\inf\big\{s\in\bR: \psi^*(s,\delta)>u\big\}$
denotes its generalized inverse.
In that case, we say that $f$ is a $(d,\psi)$-process.
For example if $f$ is sub-Gamma, that is:
\begin{equation}
  \label{eq:sub_gamma}
  \psi(\lambda,\delta)\leq \frac{\nu \lambda^2 \delta^2}{2(1-c\lambda \delta)}\,,
\end{equation}
we obtain,
\begin{equation}
  \label{eq:sub_gamma_tail}
  \ell_u(x,y) \leq \big(c u + \sqrt{2\nu u}\big) d(x,y)\,.
\end{equation}
The generality of Eq.~\ref{eq:sub_gamma} makes it convenient to derive bounds
for a wide variety of processes beyond Gaussian processes,
as we see for example in Section~\ref{sec:gp2}.

\subsection{A Tree of Successive Discretizations}

As stated in the introduction, our strategy to obtain confidence intervals
for stochastic processes is by successive discretization of $\cX$.
We define a notion of tree that will be used for this purpose.
A set $\cT=\big(\cT_h\big)_{h\geq 0}$ where $\cT_h\subset\cX$ for $h\geq 0$ is a tree
with parent relationship $p:\cX\to\cX$, when for all $x\in \cT_{h+1}$ its parent is given by
$p(x)\in \cT_h$.
We denote by $\cT_{\leq h}$ the set of the nodes of $\cT$ at depth lower than $h$:
$\cT_{\leq h} = \bigcup_{h'\leq h} \cT_h'$.
For $h\geq 0$ and a node $x\in \cT_{h'}$ with $h\leq h'$,
we also denote by $p_h(x)$ its parent at depth $h$,
that is $p_h(x) = p^{h'-h}(x)$
and we note $x\succ s$ when $s$ is a parent of $x$.
To simplify the notations in the sequel,
we extend the relation $p_h$ to $p_h(x)=x$ when $x\in\cT_{\leq h}$.

We now introduce a powerful inequality bounding the supremum of the difference of $f$
between a node and any of its descendent in $\cT$,
provided that $\abs{\cT_h}$ is not excessively large.

\begin{theorem}[Generic Chaining Upper Bound]
  \label{thm:chaining}
  Fix any $u>0$, $a>1$ and $\big(n_h\big)_{h\in\bN}$ an increasing sequence of integers.
  Set $u_i=u+n_i+\log\big(i^a\zeta(a)\big)$
  where $\zeta$ is the Riemann zeta function.
  Then for any tree $\cT$ such that $\abs{\cT_h}\leq e^{n_h}$,
  \[\forall h\geq 0, \forall s\in\cT_h,~ \sup_{x\succ s} f(x)-f(s) \leq \omega_h\,,\]
  holds with probability at least $1-e^{-u}$,
  where,
  \[\omega_h = \sup_{x\in\cX} \sum_{i> h} \ell_{u_i}\big(p_i(x), p_{i-1}(x)\big)\,.\]
\end{theorem}

The full proof of the theorem can be found in Appendix~\ref{sec:proof_chaining}.
It relies on repeated application of the union bound over the $e^{n_i}$ pairs $\big(p_i(x),p_{i-1}(x)\big)$.

Now, if we look at $\cT_h$ as a discretization of $\cX$
where a point $x\in\cX$ is approximated by $p_h(x)\in\cT_h$,
this result can be read in terms of discretization error,
as stated in the following corollary.

\begin{corollary}[Discretization error of $\cT_h$]
  \label{cor:chaining}
  Under the assumptions of Theorem~\ref{thm:chaining}
  with $\cX=\cT_{\leq h_0}$ for $h_0$ large enough, we have that,
  \[\forall h, \forall x\in\cX,~ f(x)-f(p_h(x)) \leq \omega_h\,,\]
  holds with probability at least $1-e^{-u}$.  
\end{corollary}

\subsection{Geometric Interpretation for $(d,\psi)$-processes}
\label{sec:psi_process}

The previous inequality suggests that to obtain a good upper bound on the discretization error,
one should take $\cT$ such that $\ell_{u_i}(p_i(x),p_{i-1}(x))$
is as small as possible for every $i>0$ and $x\in\cX$.
We specify what it implies for $(d,\psi)$-processes.
In that case, we have:
\[\omega_h \leq \sup_{x\in\cX} \sum_{i>h} \psi^{*-1}\Big(u_i,d\big(p_i(x),p_{i-1}(x)\big)\Big)\,.\]
Writing $\Delta_i(x)=\sup_{x'\succ p_i(x)}d(x',p_i(x))$
the $d$-radius of the ``cell'' at depth $i$ containing $x$,
we remark that $d(p_i(x),p_{i-1}(x))\leq \Delta_{i-1}(x)$,
that is:
\[
  \omega_h \leq \sup_{x\in\cX} \sum_{i>h} \psi^{*-1}\big(u_i,\Delta_{i-1}(x)\big)\,.
\]
In order to make this bound as small as possible,
one should spread the points of $\cT_h$ in $\cX$
so that $\Delta_h(x)$ is evenly small,
while satisfying the requirement $\abs{\cT_h}\leq e^{n_h}$.
Let $\Delta = \sup_{x,y\in\cX}d(x,y)$ and $\epsilon_h=\Delta 2^{-h}$,
and define an $\epsilon$-net as a set $T\subseteq \cX$ for which $\cX$ is covered by $d$-balls
of radius $\epsilon$ with center in $T$.
Then if one takes $n_h=2\log N(\cX,d,\epsilon_h)$, twice the metric entropy of $\cX$,
that is the logarithm of the minimal $\epsilon_h$-net,
we obtain with probability at least $1-e^{-u}$ that
$\forall h\geq 0, \forall s\in\cT_h$\,:
\begin{equation}
  \label{eq:classical_chaining}
  \sup_{x\succ s}f(x)-f(s) \leq \sum_{i>h} \psi^{*-1}(u_i, \epsilon_i)\,,
\end{equation}
where $u_i= u+2\log N(\cX,d,\epsilon_i)+\log(i^a\zeta(a))$.
The tree $\cT$ achieving this bound consists in computing a minimal $\epsilon$-net at each depth,
which can be done efficiently by Algorithm~\ref{alg:greedy_cover}
if one is satisfied by an almost optimal heuristic
which exhibits an approximation ratio of $\max_{x\in\cX} \sqrt{\log \log \abs{\cB(x,\epsilon)}}$,
as discussed in Appendix~\ref{sec:greedy_cover}.
This technique is often called \emph{classical chaining} \citep{Dudley1967}
and we note that an implementation appears in \cite{Contal2015} on real data.
However the upper bound in Eq.~\ref{eq:classical_chaining}
is not tight as for instance with a Gaussian process indexed by an ellipsoid,
as discussed in Section~\ref{sec:gp}.
We will present later in Section~\ref{sec:lower_bound} an algorithm to compute a tree $\cT$
in quadratic time leading to both a lower and upper bound on $\sup_{x\succ s}f(x)-f(s)$
when $f$ is a Gaussian process.

The previous inequality is particularly convenient when we know a bound
on the growth of the metric entropy of $(\cX,d)$, as stated in the following corollary.

\begin{corollary}[Sub-Gamma process with metric entropy bound]
  \label{cor:subgamma_bigoh}
  If $f$ is sub-Gamma and there exists $R,D\in\bR$ such that for all $\epsilon>0$,
  $N(\cX,d,\epsilon) \leq (\frac R \epsilon)^D$, then with probability at least $1-e^{-u}$\,:
  \[\forall h\geq 0,\forall s\in\cT_h,~ \sup_{x\succ s}f(x)-f(s) =\cO\Big(\big(c(u + D h)+\sqrt{\nu(u+Dh)}\big) 2^{-h}\Big)\,.\]
\end{corollary}
\begin{proof}
  With the condition on the growth of the metric entropy,
  we obtain $u_i = \cO\big(u+D\log R + D i\big)$.
  With Eq.~\ref{eq:classical_chaining} for a sub-Gamma process we get,
  knowing that $\sum_{i=h}^\infty i 2^{-i} =\cO\big(h 2^{-h}\big)$
  and $\sum_{i=h}^\infty \sqrt{i}2^{-i}=\cO\big(\sqrt{h}2^{-h}\big)$,
  that $\omega_h = \cO\Big(\big(c (u+D h) + \sqrt{\nu(u+D h)}\big)2^{-h}\Big)$.
\end{proof}

Note that the conditions of Corollary~\ref{cor:subgamma_bigoh} are fulfilled
when $\cX\subset [0,R]^D$ and there is $c\in\bR$ such that for all $x,y\in\cX,~d(x,y) \leq c\norm{x-y}_2$,
by simply cutting $\cX$ in hyper-cubes of side length $\epsilon$.
We also remark that this condition is very close to the near-optimality dimension of the metric space $(\cX,d)$
defined in \cite{Bubeck2011}.
However our condition constraints the entire search space $\cX$
instead of the near-optimal set $\cX_\epsilon = \big\{ x\in\cX: f(x)\geq \sup_{x^\star\in\cX}f(x^\star)-\epsilon\big\}$.
Controlling the dimension of $\cX_\epsilon$ may allow to obtain an exponential decay of the regret
in particular deterministic function $f$ with a quadratic behavior near its maximum.
However, up to our knowledge no progress has been made in this direction for stochastic processes
without constraining its behavior around the maximum.
A reader interested in this subject may look at
the recent work by \cite{Grill2015} on smooth and noisy functions with unknown smoothness,
and the works by \cite{Freitas2012} or \cite{Wang2014b}
on Gaussian processes without noise and a quadratic local behavior.

\section{Regret Bounds for Bandit Algorithms}
\label{sec:regret}

Now we have a tool to discretize $\cX$ at a certain accuracy,
we show here how to derive an optimization strategy on $\cX$.

\subsection{High Confidence Intervals}
Assume that given $i-1$ observations $Y_{i-1}=(y_1,\dots,y_{i-1})$ at queried locations $X_{i-1}$,
we can compute $L_i(x,u)$ and $U_i(x,u)$ for all $u>0$ and $x\in\cX$, such that:
\[ \Pr\Big[ f(x) \in \big(L_i(x,u), U_i(x,u)\big) \Big] \geq 1-e^{-u}\,.\]
Then for any $h(i)>0$ that we will carefully choose later,
we obtain by a union bound on $\cT_{h(i)}$ that:
\[ \Pr\Big[ \forall x\in\cT_{h(i)},~ f(x) \in \big(L_i(x,u+n_{h(i)}), U_i(x,u+n_{h(i)})\big) \Big] \geq 1-e^{-u}\,.\]
And by an additional union bound on $\bN$ that:
\begin{equation}
  \label{eq:ucb}
  \Pr\Big[ \forall i\geq 1, \forall x\in\cT_{h(i)},~ f(x) \in \big(L_i(x,u_i), U_i(x,u_i)\big) \Big] \geq 1-e^{-u}\,,
\end{equation}
where $u_i=u+n_{h(i)}+\log\big(i^a\zeta(a)\big)$ for any $a>1$ and $\zeta$ is the Riemann zeta function.
Our \emph{optimistic} decision rule for the next query is thus:
\begin{equation}
  \label{eq:argmax}
  x_i \in \argmax_{x\in\cT_{h(i)}} U_i(x,u_i)\,.
\end{equation}
Combining this with Corollary~\ref{cor:chaining}, we are able to prove the following bound
linking the regret with $\omega_{h(i)}$ and the width of the confidence interval.

\begin{theorem}[Generic Regret Bound]
  \label{thm:regret_bound}
  When for all $i\geq 1$, $x_i \in \argmax_{x\in \cT_{h(i)}} U_i(x,u_i)$
  we have with probability at least $1- 2 e^{-u}$:
  \[ R_t = t \sup_{x\in\cX} f(x)-\sum_{i=1}^t f(x_i) \leq \sum_{i=1}^t\Big\{ \omega_{h(i)} + U_i(x_i,u_i)-L_i(x_i,u_i)\Big\}\,.\]
\end{theorem}

\begin{proof}
  Using Theorem~\ref{thm:chaining} we have that,
  \[\forall h\geq 0,\,\sup_{x\in\cX}f(x) \leq \omega_h+\sup_{x\in\cX}f(p_h(x))\,,\]
  holds with probability at least $1-e^{-u}$.
  Since $p_{h(i)}(x) \in \cT_{h(i)}$ for all $x\in\cX$,
  we can invoke Eq.~\ref{eq:ucb}\,:
  \[\forall i\geq 1\,~ \sup_{x\in\cX} f(x)-f(x_i) \leq \omega_{h(i)}+\sup_{x\in\cT_{h(i)}}U_i(x,u_i)-L_i(x_i,u_i)\,,\]
  holds with probability at least $1-2e^{-u}$.
  Now by our choice for $x_i$, $\sup_{x\in\cT_{h(i)}}U_i(x,u_i) = U_i(x_i,u_i)$,
  proving Theorem~\ref{thm:regret_bound}.
\end{proof}

In order to select the level of discretization $h(i)$ to reduce the bound on the regret,
it is required to have explicit bounds on $\omega_i$ and the confidence intervals.
For example by choosing 
\[h(i)=\min\Big\{i:\bN: \omega_i \leq \sqrt{\frac{\log i}{i}} \Big\}\,,\]
we obtain $\sum_{i=1}^t \omega_{h(i)} \leq 2\sqrt{t\log t}$ as shown later.
The performance of our algorithm is thus linked with the decrease rate of $\omega_i$,
which characterizes the ``size'' of the optimization problem.
We first study the case where $f$ is distributed as a Gaussian process,
and then for a sum of squared Gaussian processes.

\subsection{Results for Gaussian Processes}
\label{sec:gp}
The problem of regret minimization where $f$ is sampled from a Gaussian process
has been introduced by \cite{Srinivas2010} and \cite{grunewalder2010}.
Since then, it has been extensively adapted to various settings of Bayesian optimization
with successful practical applications.
In the first work the authors address the cumulative regret
and assume that either $\cX$ is finite or that the samples of the process are Lipschitz
with high probability, where the distribution of the Lipschitz constant has Gaussian tails.
In the second work the authors address the simple regret without noise and with known horizon,
they assume that the canonical pseudo-metric $d$ is bounded by a given power of the supremum norm.
In both works they require that the input space is a subset of $\bR^D$.
The analysis in our paper permits to derive similar bounds in a nonparametric fashion
where $(\cX,d)$ is an arbitrary metric space.
Note that if $(\cX,d)$ is not totally bounded, then the supremum of the process is infinite with probability one,
so is the regret of any algorithm.

\paragraph{Confidence intervals and information gain.}
First, $f$ being distributed as a Gaussian process,
it is easy to derive confidence intervals given a set of observations.
Writing $\mat{Y}_i$ the vector of noisy values at points in $X_i$,
we find by Bayesian inference \citep{Rasmussen2006} that:
\[\Pr\Big[ \abs{f(x)-\mu_i(x)} \geq \sigma_i(x)\sqrt{2u}\Big] < e^{-u}\,,\]
for all $x\in\cX$ and $u>0$, where:
\begin{align}
  \label{eq:mu}
  \mu_i(x) &= \mat{k}_i(x)^\top \mat{C}_i^{-1}\mat{Y}_i\\
  \label{eq:sigma}
  \sigma_i^2(x) &= k(x,x) - \mat{k}_i(x)^\top \mat{C}_i^{-1} \mat{k}_i(x)\,,
\end{align}
where $\mat{k}_i(x) = [k(x_j, x)]_{x_j \in X_i}$ is the covariance vector between $x$ and $X_i$,
$\mat{C}_i = \mat{K}_i + \eta^2 \mat{I}$,
and $\mat{K}_i=[k(x,x')]_{x,x' \in X_i}$ the covariance matrix
and $\eta^2$ the variance of the Gaussian noise.
Therefore the width of the confidence interval in Theorem~\ref{thm:regret_bound}
can be bounded in terms of $\sigma_{i-1}$:
\[U_i(x_i,u_i)-L_i(x_i,u_i) \leq 2\sigma_{i-1}(x_i)\sqrt{2u_i}\,.\]
Furthermore it is proved in \cite{Srinivas2012} that the sum of the posterior variances
at the queried points $\sigma_{i-1}^2(x_i)$ is bounded in terms of information gain:
\[\sum_{i=1}^t \sigma_{i-1}^2(x_i) \leq c_\eta \gamma_t\,,\]
where $c_\eta=\frac{2}{\log(1+\eta^{-2})}$
and $\gamma_t = \max_{X_t\subseteq\cX:\abs{X_t}=t} I(X_t)$
is the maximum information gain of $f$ obtainable by a set of $t$ points.
Note that for Gaussian processes,
the information gain is simply $I(X_t)=\frac 1 2 \log\det(\mat{I}+\eta^{-2}\mat{K}_t)$.
Finally, using the Cauchy-Schwarz inequality and the fact that $u_t$ is increasing we have
with probability at least $1- 2 e^{-u}$:
\begin{equation}
  \label{eq:gp_regret}
  R_t \leq 2\sqrt{2 c_\eta t u_t \gamma_t} + \sum_{i=1}^t \omega_{h(i)}\,.
\end{equation}
The quantity $\gamma_t$ heavily depends on the covariance of the process.
On one extreme, if $k(\cdot,\cdot)$ is a Kronecker delta,
$f$ is a Gaussian white noise process and $\gamma_t=\cO(t)$.
On the other hand \cite{Srinivas2012} proved the following inequalities for widely used covariance functions
and $\cX\subset \bR^D$:
\begin{itemize}
\item linear covariance $k(x,y)=x^\top y$, $\gamma_t=\cO\big(D \log t\big)$.
\item squared exponential covariance $k(x,y)=e^{-\frac 1 2 \norm{x-y}_2^2}$, $\gamma_t=\cO\big((\log t)^{D+1}\big)$.
\item Mat\'ern covariance, $k(x,y)=\frac{2^{p-1}}{\Gamma(p)}\big(\sqrt{2p}\norm{x-y}_2\big)^p K_p\big(\sqrt{2p}\norm{x-y}_2\big)$,
  where $p>0$ and $K_p$ is the modified Bessel function,
  $\gamma_t=\cO\big( (\log t) t^a\big)$, with $a=\frac{D(D+1)}{2p+D(D+1)}<1$ for $p>1$.
\end{itemize}

\paragraph{Bounding $\omega_h$ with the metric entropy.}
We now provide a policy to choose $h(i)$ minimizing the right hand side of Eq.\ref{eq:gp_regret}.
When an explicit upper bound on the metric entropy of the form
$\log N(\cX,d,\epsilon)\leq \cO(-D \log \epsilon)$ holds,
we can use Corollary~\ref{cor:subgamma_bigoh} which gives:
\[\omega_h\leq\cO\big(\sqrt{u+D h}2^{-h}\big)\,.\]
This upper bound holds true in particular for Gaussian processes with $\cX\subset[0,R]^D$
and for all $x,y\in\cX$, $d(x,y) \leq \cO\big(\norm{x-y}_2\big)$.
For stationary covariance this becomes $k(x,x)-k(x,y)\leq \cO\big(\norm{x-y}_2\big)$
which is satisfied for the usual covariances used in Bayesian optimization such as
the squared exponential covariance
or the Mat\'ern covariance with parameter $p\in\big(\frac 1 2, \frac 3 2, \frac 5 2\big)$.
For these values of $p$ it is well known that $k(x,y)=h_p\big(\sqrt{2p}\norm{x-y}_2\big) \exp\big(-\sqrt{2p}\norm{x-y}_2\big)$,
with $h_{\frac 1 2}(\delta)=1$, $h_{\frac 3 2}(\delta)=1+\delta$ and $h_{\frac 5 2}(\delta)=1+\delta+\frac 1 3 \delta^2$.
Then we see that is suffices to choose $h(i)=\ceil{\frac 1 2 \log_2 i}$
to obtain $\omega_{h(i)} \leq \cO\Big( \sqrt{\frac{u+\frac 1 2 D\log i}{i}} \Big)$
and since $\sum_{i=1}^t i^{-\frac 1 2}\leq 2 \sqrt{t}$ and
$\sum_{i=1}^t \big(\frac{\log i}{i}\big)^{\frac 1 2} \leq 2\sqrt{t\log t}$,
\[R_t \leq  \cO\Big(\sqrt{t \gamma_t \log t }\Big)\,, \]
holds with high probability.
Such a bound holds true in particular for the Ornstein-Uhlenbeck process,
which was conjectured impossible in \cite{Srinivas2010} and \cite{Srinivas2012}.
However we do not know suitable bounds for $\gamma_t$ in this case
and can not deduce convergence rates.

\paragraph{Gaussian processes indexed on ellipsoids and RKHS.}
As mentioned in Section~\ref{sec:psi_process}, the previous bound on the discretization error
is not tight for every Gaussian process.
An important example is when the search space is a (possibly infinite dimensional) ellipsoid:
\[\cX=\Big\{ x\in \ell^2: \sum_{i\geq 1}\frac{x_i^2}{a_i^2} \leq 1\Big\}\,.\]
where $a\in\ell^2$,
and $f(x) = \sum_{i\geq 1}x_ig_i$ with $g_i\iid \cN(0,1)$,
and the pseudo-metric $d(x,y)$ coincide with the usual $\ell_2$ metric.
The study of the supremum of such processes is connected to learning error bounds
for kernel machines like Support Vector Machines,
as a quantity bounding the learning capacity of a class of functions in a RKHS,
see for example \cite{Mendelson2002}.
It can be shown by geometrical arguments that
$\E \sup_{x: d(x,s)\leq \epsilon} f(x)-f(s) \leq \cO\big(\sqrt{\sum_{i\geq 1}\min(a_i^2,\epsilon^2)}\big)\,,$
and that this supremum exhibits $\chi^2$-tails around its expectation,
see for example \cite{Boucheron2013} and \cite{Talagrand2014}.
This concentration is not grasped by Corollary~\ref{cor:subgamma_bigoh},
it is required to leverage the construction of Section~\ref{sec:lower_bound}
to get a tight estimate.
Therefore the present work forms a step toward efficient and practical online model selection
in such classes in the spirit of \cite{Rakhlin2014} and \cite{Gaillard2015}.

\subsection{Results for Quadratic Forms of Gaussian Processes}
\label{sec:gp2}

The preeminent model in Bayesian optimization is by far the Gaussian process.
Yet, it is a very common task to attempt minimizing a regret on functions which
does not look like Gaussian processes.
Consider the typical cases where $f$ has the form of a mean square error
or a Gaussian likelihood.
In both cases, minimizing $f$ is equivalent to minimize a sum of squares,
which we can not assume to be sampled from a Gaussian process.
To alleviate this problem, we show that this objective fits in our generic setting.
Indeed, if we consider that $f$ is a sum of squares of Gaussian processes,
then $f$ is sub-Gamma with respect to a natural pseudo-metric.
In order to match the challenge of maximization, we will precisely take the opposite.
In this particular setting we allow the algorithm
to observe directly the noisy values of the \emph{separated} Gaussian processes,
instead of the sum of their square.
To simplify the forthcoming arguments, we will choose independent and identically distributed processes,
but one can remove the covariances between the processes by Cholesky decomposition of the covariance matrix,
and then our analysis adapts easily to processes with non identical distributions.

\paragraph{The stochastic smoothness of squared GP.}
Let $f=-\sum_{j=1}^N g_j^2(x)$,
where $\big(g_j\big)_{1\le j\le N}$ are independent centered Gaussian processes $g_j\iid\cGP(0,k)$
with stationary covariance $k$ such that $k(x,x)=\kappa$ for every $x\in\cX$.
We have for $x,y\in\cX$ and $\lambda<(2\kappa)^{-1}$:
\[\log\E e^{\lambda(f(x)-f(y))} = -\frac{N}{2}\log\Big(1-4\lambda^2(\kappa^2-k^2(x,y))\Big)\,. \]
Therefore with $d(x,y)=2\sqrt{\kappa^2-k^2(x,y)}$ and $\psi(\lambda,\delta)=-\frac{N}{2}\log\big(1-\lambda^2\delta^2\big)$,
we conclude that $f$ is a $(d,\psi)$-process.
Since $-\log(1-x^2) \leq \frac{x^2}{1-x}$ for $0\leq x <1$,
which can be proved by series comparison,
we obtain that $f$ is sub-Gamma with parameters $\nu=N$ and $c=1$.
Now with Eq.~\ref{eq:sub_gamma_tail},
\[\ell_u(x,y)\leq (u+\sqrt{2 u N})d(x,y)\,.\]
Furthermore, we also have that $d(x,y)\leq \cO(\norm{x-y}_2)$ for $\cX\subseteq \bR^D$
and standard covariance functions including
the squared exponential covariance or the Mat\'ern covariance with parameter $p=\frac 3 2$ or $p=\frac 5 2$.
Then Corollary~\ref{cor:subgamma_bigoh} leads to:
\begin{equation}
  \label{eq:omega_gp2}
  \forall i\geq 0,~ \omega_i \leq \cO\Big( u+D i + \sqrt{N(u+D i)}2^{-i}\Big)\,. 
\end{equation}

\paragraph{Confidence intervals for squared GP.}
As mentioned above, we consider here that we are given separated noisy observations $\mat{Y}_i^j$
for each of the $N$ processes.
Deriving confidence intervals for $f$ given $\big(\mat{Y}_i^j\big)_{j\leq N}$
is a tedious task since the posterior processes $g_j$ given $\mat{Y}_i^j$
are not standard nor centered.
We propose here a solution based directly on a careful analysis of Gaussian integrals.
The proof of the following technical lemma can be found in Appendix~\ref{sec:gp2_tail}.

\begin{lemma}[Tails of squared Gaussian]
\label{lem:gp2_tail}
  Let $X\sim\cN(\mu,\sigma^2)$ and $s>0$. We have:
  \[\Pr\Big[ X^2 \not\in \big(l^2, u^2\big)\Big] < e^{-s^2}\,,\]
  for $u=\abs{\mu}+\sqrt{2} \sigma s$
  and $l=\max\big(0,\abs{\mu}-\sqrt{2}\sigma s\big)$.
\end{lemma}

Using this lemma, we compute the confidence interval for $f(x)$
by a union bound over $N$.
Denoting $\mu_i^j$ and $\sigma_i^j$ the posterior expectation and deviation
of $g_j$ given $\mat{Y}_i^j$ (computed as in Eq.~\ref{eq:mu} and \ref{eq:sigma}),
the confidence interval follows for all $x\in\cX$:
\begin{equation}
  \label{eq:gp2_ci}
  \Pr\Big[ \forall j\leq m,~ g_j^2(x) \in \big( L_i^j(x,u), U_i^j(x,u) \big)\Big] \geq 1- e^{-u}\,,
\end{equation}
where
\begin{align*}
  U_i^j(x,u) &= \Big(\abs{\mu_i^j(x)}+\sqrt{2(u+\log N)} \sigma_{i-1}^j(x)\Big)^2\\
  \text{ and } L_i^j(x,u) &= \max\Big(0, \abs{\mu_i^j(x)}-\sqrt{2(u+\log N)} \sigma_{i-1}^j(x)\Big)^2\,.
\end{align*}
We are now ready to use Theorem~\ref{thm:regret_bound} to control $R_t$
by a union bound for all $i\in\bN$ and $x\in\cT_{h(i)}$.
Note that under the event of Theorem~\ref{thm:regret_bound},
we have the following:
\[\forall j\leq m, \forall i\in\bN, \forall x\in\cT_{h(i)},~ g_j^2(x) \in \big(L_i^j(x,u_i), U_i^j(x,u_i)\big)\,,\]
Then we also have:
\[\forall j\leq m, \forall i\in\bN, \forall x\in\cT_{h(i)},~ \abs{\mu_i^j(x)} \leq \abs{g_j(x)}+\sqrt{2(u_i+\log N)}\sigma_{i-1}^j(x)\,,\]
Since $\mu_0^j(x)=0$, $\sigma_0^j(x)=\kappa$ and $u_0\leq u_i$
we obtain $\abs{\mu_i^j(x)} \leq \sqrt{2(u_i+\log N)}\big(\sigma_{i-1}^j(x)+\kappa\big)$.
Therefore Theorem~\ref{thm:regret_bound} says with probability at least $1-2e^{-u}$:
\[R_t \leq \sum_{i=1}^t\Big\{\omega_{h(i)} + 8\sum_{j\leq N}(u_i+\log N)\big(\sigma_{i-1}^j(x)+\kappa\big)\sigma_{i-1}^j(x_i) \Big\}\,.\]
It is now possible to proceed as in Section~\ref{sec:gp} and bound the sum of posterior variances with $\gamma_t$\,:
\[R_t \leq \cO\Big( N u_t \big(\sqrt{t \gamma_t} + \gamma_t\big) + \sum_{i=1}^t \omega_{h(t)} \Big)\,.\]
As before, under the conditions of Eq.~\ref{eq:omega_gp2} and
choosing the discretization level $h(i)=\ceil{\frac 1 2 \log_2 i}$
we obtain $\omega_{h(i)}=\cO\Big(i^{-\frac 1 2} \big(u+\frac 1 2 D\log i\big)\sqrt{N}\Big)$,
and since $\sum_{i=1}^t i^{-\frac 1 2} \log i\leq 2 \sqrt{t}\log t$,
\[R_t \leq \cO\Big(N \big(\sqrt{t\gamma_t \log t}+\gamma_t\big) + \sqrt{Nt}\log t\Big)\,,\]
holds with high probability.

\section{Tightness Results for Gaussian Processes}
\label{sec:lower_bound}

We present in this section a strong result on the tree $\cT$ obtained by Algorithm~\ref{alg:tree_lb}.
Let $f$ be a centered Gaussian process $\cGP(0,k)$ with arbitrary covariance $k$.
We show that a converse of Theorem~\ref{thm:chaining} is true with high probability.

\subsection{A High Probabilistic Lower Bound on the Supremum}
We first recall that for Gaussian process we have $\psi^{*-1}(u_i,\delta)=\cO\big(\delta \sqrt{u+n_i}\big)$,
that is:
\[\forall h\geq 0, \forall s\in\cT_h,~\sup_{x\succ s}f(x)-f(s) \leq \cO\Big(\sup_{x\succ s}\sum_{i>h}\Delta_i(x) \sqrt{u+n_i}\Big)\,,\]
with probability at least $1-e^{-u}$.
For the following, we will fix for $n_i$ a geometric sequence $n_i=2^i$ for all $i\geq 1$.
Therefore we have the following upper bound:

\begin{corollary}
  Fix any $u>0$ and let $\cT$ be constructed as in Algorithm~\ref{alg:tree_lb}.
  Then there exists a constant $c_u>0$ such that, for $f\sim\cGP(0,k)$,
  \[\sup_{x\succ s} f(x)-f(s) \leq c_u \sup_{x\succ s} \sum_{i>h} \Delta_i(x)2^{\frac i 2}\,,\]
  holds for all $h\geq 0$ and $s\in\cT_h$ with probability at least $1-e^{-u}$.
\end{corollary}

To show the tightness of this result,
we prove the following probabilistic bound:
\begin{theorem}[Generic Chaining Lower Bound]
  \label{thm:lower_bound}
  Fix any $u>0$ and let $\cT$ be constructed as in Algorithm~\ref{alg:tree_lb}.
  Then there exists a constant $c_u>0$ such that, for $f\sim\cGP(0,k)$,
  \[\sup_{x\succ s} f(x)-f(s) \geq c_u \sup_{x\succ s}\sum_{i=h}^\infty \Delta_i(x)2^{\frac i 2}\,,\]
  holds for all $h\geq 0$ and $s\in\cT_h$ with probability at least $1-e^{-u}$.
\end{theorem}

The benefit of this lower bound is huge for theoretical and practical reasons.
It first says that we cannot discretize $\cX$ in a finer way that Algorithm~\ref{alg:tree_lb}
up to a constant factor.
This also means that even if the search space $\cX$ is ``smaller''
than what suggested using the metric entropy,
like for ellipsoids,
then Algorithm~\ref{alg:tree_lb} finds the correct ``size''.
Up to our knowledge, this result is the first construction of tree $\cT$
leading to a lower bound at every depth with high probability.
The proof of this theorem shares some similarity with the construction
to obtain lower bound in expectation,
see for example \cite{Talagrand2014} or \cite{Ding2011} for a tractable algorithm.

\subsection{Analysis of Algorithm~\ref{alg:tree_lb}}
Algorithm~\ref{alg:tree_lb} proceeds as follows.
It first computes $(\cT_h)_{h\geq 0}$ a succession of $\epsilon_h$-nets as in Section~\ref{sec:psi_process}
with $\epsilon_h=\Delta 2^{-h}$ where $\Delta$ is the diameter of $\cX$.
The parent of a node is set to the closest node in the upper level,
\[\forall t\in\cT_h,~ p(t) = \argmin_{s\in\cT_{h-1}} d(t,s)\,\]
Therefore we have $d(t,p(t))\leq \epsilon_{h-1}$ for all $t\in\cT_h$.
Moreover, by looking at how the $\epsilon_h$-net is computed we also have
$d(t_i,t_j) \geq \epsilon_h$ for all $t_i,t_j\in\cT_h$.
These two properties are crucial for the proof of the lower bound.

Then, the algorithm updates the tree to make it well balanced,
that is such that no node $t\in\cT_h$ has more that $e^{n_{h+1}-n_h}=e^{2^h}$ children.
We note at this time that this condition will be already satisfied in every reasonable space,
so that the complex procedure that follows is only required in extreme cases.
To force this condition, Algorithm~\ref{alg:tree_lb} starts from the leafs
and ``prunes'' the branches if they outnumber $e^{2^h}$.
We remark that this backward step is not present in the literature on generic chaining,
and is needed for our objective of a lower bound with high probability.
By doing so, it creates a node called a \emph{pruned node} which will take as children
the pruned branches.
For this construction to be tight, the pruning step has to be careful.
Algorithm~\ref{alg:tree_lb} attaches to every pruned node a value,
computed using the values of its children,
hence the backward strategy.
When pruning branches, the algorithm keeps the $e^{2^h}$ nodes with maximum values and displaces the others.
The intuition behind this strategy is to avoid pruning branches that already contain pruned node.

Finally, note that this pruning step may creates unbalanced pruned nodes
when the number of nodes at depth $h$ is way larger that $e^{2^h}$.
When this is the case, Algorithm~\ref{alg:tree_lb}
restarts the pruning with the updated tree to recompute the values.
Thanks to the doubly exponential growth in the balance condition,
this can not occur more that $\log \log \abs{\cX}$ times
and the total complexity is $\cO\big(\abs{\cX}^2\big)$.

\subsection{Computing the Pruning Values and Anti-Concentration Inequalities}
We end this section by describing the values used for the pruning step.
We need a function $\varphi(\cdot,\cdot,\cdot,\cdot)$
satisfying the following anti-concentration inequality.
For all $m\in\bN$, let $s\in\cX$ and $t_1,\dots,t_m\in\cX$ such that
$\forall i\leq m,~p(t_i)=s$ and $d(s,t_i)\leq \Delta$,
and finally $d(t_i,t_j)\geq \alpha$.
Then $\varphi$ is such that:
\begin{equation}
  \label{eq:varphi}
  \Pr\Big[\max_{i\leq m}f(t_i)-f(s) \geq \varphi(\alpha,\Delta,m,u) \Big]>1-e^{-u}\,.
\end{equation}
A function $\varphi$ satisfying this hypothesis is described in Lemma~\ref{lem:max_one_lvl}
in Appendix~\ref{sec:proof_lower_bound}.
Then the value $V_h(s)$ of a node $s\in\cT_h$ is computed with $\Delta_i(s) = \sup_{x\succ s} d(x,s)$ as:
\[V_h(s) = \sup_{x\succ s} \sum_{i>h} \varphi\Big(\frac 1 2 \Delta_h(x),\Delta_h(x),m,u\Big) \one_{p_i(x)\text{ is a pruned node}}\,.\]
The two steps proving Theorem~\ref{thm:lower_bound} are:
first, show that $\sup_{x\succ s}f(x)-f(s) \geq c_u V_h(s)$ for $c_u>0$ with probability at least $1-e^{-u}$,
second, show that $V_h(s) \geq c_u'\sup_{x\succ s}\sum_{i>h}\Delta_i(x)2^{\frac i 2}$
for $c_u'>0$.
The full proof of this theorem can be found in Appendix~\ref{sec:proof_lower_bound}.

\paragraph{Acknowledgements.}
We thank C\'edric Malherbe and Kevin Scaman for fruitful discussions.

\bibliography{../../biblio/biblio}

\appendix

\section{Algorithms to Compute an Optimal Tree}
\label{sec:algo}

\begin{algorithm}[H]
  \DontPrintSemicolon
  \KwData{$\Delta=\sup_{x,y\in\cX}d(x,y)$, $u>0$, $\varphi$ as in Eq.~\ref{eq:varphi}}
  \tcc{\hfill Forward pass: compute $\epsilon_h$-nets \hfill}
  $h \gets 0$\;
  $\cT \gets \{x_0\}$ for arbitrary $x_0\in\cX$\;
  \While{$\cT \neq \cX$}{
    $h \gets h+1$\;
    $\epsilon_h \gets 2^{-h-1}\Delta$\;
    $T_h \gets \textsc{GreedyCover}\Big(\epsilon_h, \cX\setminus \bigcup_{t\in\cT}\cB(t,\epsilon_h)\Big)$\;
    $\forall t\in T_h,~p(t) \gets \argmin_{s\in\cT} d(t,s)$\;
    $\cT \gets \cT \cup T_h$
  }
  \tcc{\hfill Backward pass: prune the tree \hfill}
  $\forall t\in\cT_h,~ V_h(t) \gets 0$\;
  \While{$h>0$}{
    \For{$s\in\cT_{h-1}$}{
      $T_s \gets \big\{t:p(t)=s\big\}$\;
      $\forall t\in T_s,~V_h(t) \gets \sup_{t':p(t')=t} V_{h+1}(t')$
      \tcp*[f]{Default value}\;
      $m \gets e^{n_h-n_{h-1}}$\;
      \If(\tcp*[f]{if the tree is not balanced}){$\abs{T_s} > m$}{
        Let $t_1,\dots,t_n\in T_s$ ordered by decreasing $V_h(t)$\;
        Create a pruned node $t$ and set $p(t)\gets s$\;
        $\forall i\geq m, \forall t'\,s.t.\,p(t')=t_j,~ p(t')\gets t$\;
        \eIf{$\abs{\big\{t':p(t')=t\big\}}\leq e^{n_{h+1}-n_h}$}{
          $\Delta_h \gets \sup_{x\succ t}d(x,t)$
          \tcp*[f]{Update the value of the pruned node}\;
          $u_h \gets u+n_h+h\log 2$\;
          $V_h(t) \gets \sup_{t':p(t')=t} V_{h+1}(t') + \varphi\Big(\frac 1 2 \Delta_h, \Delta_h, m, u_h\Big)$\;
        }{
          Restart the pruning
          \tcp*[f]{Can not occur more that $\log \log \abs{\cX}$ times}
        }
      }
    }
  }
  return $\cT$\;
  \caption{Computing a tree $\cT$ for $(d,\psi)$-processes}
  \label{alg:tree_lb}
\end{algorithm}

\begin{algorithm}[H]
  \DontPrintSemicolon
  $T\gets \emptyset$\;
  \While{$\cX \neq \emptyset$}{
    $x \gets \argmax_{x\in\cX} \abs{\big\{x'\in\cB(x,\epsilon)\big\}}$\;
    $T \gets T \cup \{x\}$\;
    $\cX \gets \cX \setminus \cB(x,\epsilon)$\;
  }
  \Return $T$\;
  \caption{\textsc{GreedyCover}($\epsilon$, $\cX$)}
  \label{alg:greedy_cover}
\end{algorithm}

\section{Proof of Theorem~\ref{thm:chaining} (Generic Chaining Upper Bound)}
\label{sec:proof_chaining}

We give here the proof of Theorem~\ref{thm:chaining} which upper bound
the supremum $\sup_{x\succ s}f(x)-f(s)$ in terms of $\omega_h$.

\begin{proof}
  For any $s\in\cT_h$ and any $x\succ s$, $f(x)-f(s) = \sum_{i>h} f(p_i(x))-f(p_{i-1}(x))$.
  Now by definition of $\ell_u$ we have:
  \[\Pr\Big[f(p_i(x))-f(p_{i-1}(x)) \geq \ell_{u_i}\big(p_i(x),p_{i-1}(x)\big)\Big] < e^{-u_i}\,.\]
  Thanks to the tree structure $\abs{\Big\{\big(p_i(x),p_{i-1}(x)\big) : x\in\cX\Big\}} \leq e^{n_i}$.
  By a union bound we have:
  \[\Pr\Big[\exists x\in\cX,\, f(p_i(x))-f(p_{i-1}(x)) > \ell_{u_i}\big(p_i(x),p_{i-1}(x)\big)\Big] < e^{n_i}e^{-u_i}\,.\]
  With an other union bound over $i\geq 0$,
  if we denote by $E^c$ the following event: 
  \[E^c = \Big\{\exists i> 0, \exists x\in\cX,\, f(p_i(x))-f(p_{i-1}(x)) > \ell_{u_i}\big(p_i(x),p_{i-1}(x)\big)\Big\}\,,\]
  we have $\Pr[E^c] < \sum_{i\geq 0}e^{n_i-u_i}$.
  By setting $u_i = u + n_i + \log\big(i^a \zeta(a)\big)$ for $a>1$
  we have $\Pr[E^c] < e^{-u}$,
  that is $\Pr\Big[\sum_{i>h} f(p_i(x))-f(p_{i-1}(x)) \geq \sum_{i>h}\ell_{u_i}\big(p_i(x),p_{i-1}(x)\big)\Big]<e^{-u}$.
\end{proof}

\section{Analysis of \textsc{GreedyCover}}
\label{sec:greedy_cover}
\paragraph{Approximation radio.}
The exact computation of an optimal $\epsilon$-cover is \NP-hard.
We demonstrate here how to build in practice a near-optimal $\epsilon$-cover using a greedy algorithm on graph.
First, remark that for any fixed $\epsilon$ we can define a graph $\cG$ where the nodes are the elements of $\cX$
and there is an edge between $x$ and $y$ if and only if $d(x,y)\leq \epsilon$.
The size of this construction is $\cO(\abs{\cX}^2)$.
The sparse structure of the underlying graph can be exploited to get an efficient representation.
The problem of finding an optimal $\epsilon$-cover reduces to the problem
of finding a minimal dominating set on $\cG$.
We can therefore use the greedy Algorithm~\ref{alg:greedy_cover} which enjoys an approximation factor of $\log d_\mathrm{max}(\cG)$,
where $d_\mathrm{max}(\cG)$ is the maximum degree of $\cG$, which is equal to $\max_{x\in\cX}\abs{\cB(x,\epsilon)}$.
An interested reader may see for example \cite{Johnson1973} for a proof of \NP-hardness and approximation results.
This construction leads to an additional (almost constant) term of $\max_{x\in\cX}\sqrt{\log \log \abs{\cB(x,\epsilon)}}$
in the right-hand side of Eq.~\ref{eq:classical_chaining}.
Finally, note that this approximation is optimal unless $\P=\NP$ as shown in \cite{Raz1997}.

\paragraph{Computation on a compact space $\cX$.}
Even if all the theoretical analysis of this paper assumes that $\cX$ is finite for measurability reasons,
it is not satisfying from a numerical point of view.
We show here that if the search space $\cX$ is a compact,
then there is a way to reduce computations to the finite case.
First remark that is $(\cX,d)$ is compact,
then there exists a uniform distribution $\mu$ on $\cX$.
The following lemma describes the probability to get an $\epsilon$-net via uniform sampling in $\cX$.

\begin{lemma}[Covering with uniform sampling]
  Let $\mu$ be a uniform distribution on $\cX$,
  and $m=N(\cX,d,\epsilon)$,
  and $X_n=(x_1,\dots,x_n)$ be $n$ points distributed independently according to $\mu$
  with $n\geq m(\log m+u)$.
  Then with probability at least $1-e^{-u}$, $X_n$ is a $2\epsilon$-net of $\cX$.
\end{lemma}
\begin{proof}
  Let $T$ be an $\epsilon$-net on $\cX$ of cardinality $\abs{T}=m$.
  Then the probability $P^c$ that it exists $t\in T$ such that $\min_{i\leq n} d(t,x_i)>\epsilon$
  is less than:
  \[P^c \leq \sum_{t\in T} \Pr\Big[\forall i\leq n,~ x_i \not\in \cB(t,\epsilon)\Big]\,.\]
  Since $\mu$ attributes an equal probability mass for every ball of radius $\epsilon$,
  $P^c \leq m \Big(\frac{m-1}{m}\Big)^n$.
  With $\log\frac{m}{m-1}\geq \frac 1 m$,
  we have for $n\geq m(\log m+u)$ that,
  \[P^c \leq e^{-u}\,.\]
  By the triangle inequality, with probability at least $1-e^{-u}$, $X_n$ is $2\epsilon$-net.
\end{proof}

Therefore when we want to compute an $\epsilon$-net on a compact $\cX$,
an efficient way is to first sample $X_n=(x_1,\dots,x_n)$ uniformly
with $n\geq m(\log m+u)$ and $m=N(\cX,d,\frac 1 4 \epsilon)$,
which gives an $\frac 1 2 \epsilon$-net with probability at least $1-e^{-u}$.
Then running \textsc{GreedyCover}$\big(\frac 1 2 \epsilon, X_n\big)$
outputs an $\epsilon$-net of $\cX$ with probability at least $1-e^{-u}$.

\section{Proof of Lemma~\ref{lem:gp2_tail} (Tails of Squared Gaussian)}
\label{sec:gp2_tail}
We provide here the proof of Lemma~\ref{lem:gp2_tail} which obtains confidence interval
on squared Gaussian variables.
We actually prove a slightly stronger result which improves the tightness on the confidence interval,
but is not used by our theoretical analysis.
\begin{proof}
  Let $X\sim\cN(\mu,\sigma^2)$ with $\mu\geq 0$ without loss of generality.
  Write $\erf(a)=\frac{2}{\sqrt{\pi}}\int_0^a e^{-t^2} \diff t$ and $\erfc(a)=1-\erf(a)$.
  For all $0<l<u\in\bR$ we have:
  \begin{align*}
    \Pr\Big[X^2 \not\in (l,u) \Big] &= \Pr\Big[X \not\in (l,u)\cup(-u,-l) \Big]\\
    &= \frac 1 2 \Big( \erfc\Big(\frac{u-\mu}{\sqrt{2}\sigma}\Big)
                     + \erfc\Big(\frac{u+\mu}{\sqrt{2}\sigma}\Big)
                     + \erf\Big(\frac{\mu+l}{\sqrt{2}\sigma}\Big)
                     - \erf\Big(\frac{\mu-l}{\sqrt{2}\sigma}\Big) \Big)\,.
 \end{align*}
 Fix $s>0$ and $u=\mu+\sqrt{2}\sigma s$.
 If $l \leq \mu-\sqrt{2}\sigma s$, which means $s<\mu(\sqrt{2}\sigma)^{-1}$, we get:
 \[\Pr\big[X^2 \not\in (l^2,u^2) \big] \leq \frac 1 2 \Big( \erfc(s)
                     + \erfc\big(\sqrt{2}\mu\sigma^{-1}+s\big)
                     + \erf\big(\sqrt{2}\mu\sigma^{-1}-s\big)
                     - \erf(s) \Big)\,.\]
 Remarking that $\erfc\big(\sqrt{2}\mu\sigma^{-1}+s\big)+\erf\big(\sqrt{2}\mu\sigma^{-1}-s\big)\leq 1$,
 we obtain:
 \[\Pr\big[X^2\not\in(l^2,u^2)\big] \leq \erfc(s)\,.\]
 Now for $s>\mu(\sqrt{2}\sigma)^{-1}$,
 if $l\leq \sqrt{2}\sigma \erf^{-1}\Big(\frac 1 2 \erf(\sqrt{2}\mu\sigma^{-1}+s)-\frac 1 2 \erf(s)\Big)$
 we have that $\erf\Big(\frac{\mu+l}{\sqrt{2}\sigma}\Big)-\erf\Big(\frac{\mu-l}{\sqrt{2}\sigma}\Big) \leq 2\erf\big(\frac{l}{\sqrt{2}\sigma}\big) \leq \erf(\sqrt{2}\mu\sigma^{-1}+s)-\erf(s)$.
 Therefore we also get:
 \[\Pr\Big[X^2 \not\in (l^2,u^2) \Big] \leq \erfc(s)\,.\]
 We finish the proof of Lemma~\ref{lem:gp2_tail} by the standard inequality $\erfc(s)\leq e^{-s^2}$.
\end{proof}

\section{Proof of Theorem~\ref{thm:lower_bound} (Generic Chaining Lower Bound)}
\label{sec:proof_lower_bound}

In this section we provide the proof of the high probabilistic lower bound
obtained via Algorithm~\ref{alg:tree_lb}.
The proof is given for $f$ being a Gaussian process.
We note that the result remains valid for other stochastic processes
as long as Lemma~\ref{lem:max_normal} and \ref{lem:comparison} hold.

\subsection{Probabilistic Tools for Gaussian Processes}
We first prove a probabilistic bound on independent Gaussian variables
and then show that a similar bound holds for $f$ via a comparison inequality.

\begin{lemma}[Anti-concentration for independent Gaussian variables]
  \label{lem:max_normal}
  Let $(N_i)_{i\leq m}$ be $m$ independent standard normal variables.
  For $m \geq 2.6 u$ we have with probability at least $1-e^{-u}$ that:
  \[\max_{i\leq m}N_i \geq \sqrt{\log\frac{m}{2.6 u}}\,.\]
\end{lemma}
\begin{proof}
  With $N_i \iid \cN(0,1)$ for all $i\leq m$ we obtain for all $\lambda\in\bR$:
  \begin{align*}
    \Pr\Big[\max_{i\leq m} N_i \geq \lambda\Big] &= 1-\Pr[\forall i\leq m,\,N_i<\lambda]\\
      &= 1-\Pr[N_i<\lambda]^m\\
      &= 1-\Phi(\lambda)^m\,,
  \end{align*}
  where $\Phi$ is the standard normal cumulative distribution function,
  which satisfies $\Phi(\lambda) \leq 1-c_1 e^{-\lambda^2}$ with $c_1>0.38$,
  see for example \cite{Cote2012}.
  For $\lambda \leq \sqrt{\log\frac{c_1}{1-e^{-\frac u m}}}$ and $u \leq m \log\frac{1}{1-c_1}$
  we obtain $\Phi(\lambda)^m \leq e^{-u}$.
  Using that $1-e^{-x}\leq x$ for $x\geq 0$, we obtain with $u\leq c_1 m$ that:
  \[\Pr\Big[\max_{i\leq m}N_i \geq \sqrt{\log\frac{c_1 m}{u}}\Big] \geq 1-e^{-u}\,.\]
\end{proof}

The following lemma will be useful to derive anti-concentration inequalities
for non independent Gaussian variables, provided that their $L_2$ distance are large enough.
Similar results are well known if one replaces the probabilities by expectations,
see for example \cite{Ledoux1991}.

\begin{lemma}[Comparison inequality for Gaussian variables]
  \label{lem:comparison}
  Let $(X_i)_{i\leq m}$ and $(Y_i)_{i\leq m}$ be Gaussian random variables such that for all $i,j\leq m$,
  $\E(X_i-X_j)^2 \geq \E(Y_i-Y_j)^2$
  and $\E X_i^2 \geq \E Y_i^2$.
  Then we have for all $\lambda\in\bR$\,:
  \[\Pr\Big[\max_{i\leq m} X_i < \lambda-2\sigma\Big] \leq \Pr\Big[\max_{i\leq m} Y_i < \lambda\Big]\,,\]
  where $\sigma = \max_{i\leq m}(\E X_i^2)^{\frac 1 2}$.
\end{lemma}
\begin{proof}
  Let $g$ be a Rademacher variable independent of $X$ and $Y$.
  We define $\wt{X}_i = X_i + g(\sigma^2 + \E Y_i^2 - \E X_i^2)^{\frac 1 2}$
  and $\wt{Y}_i = Y_i + g \sigma$.
  With this definition, we have by simple calculus that $\E \wt{X}_i^2 = \E Y_i^2 + \sigma^2 = \E \wt{Y}_i^2$.
  Furthermore, $\E(\wt{Y}_i-\wt{Y}_j)^2 = \E(Y_i-Y_j)^2$ and $\E(\wt{X}_i-\wt{Y}_j)^2 \geq \E(X_i-X_j)^2$
  for all $i$ and $j$, that is $\E(\wt{X}_i-\wt{X}_j)^2 \geq \E(\wt{Y}_i-\wt{Y}_j)^2$.
  Combining this with the previous remark we obtain $\E[\wt{X}_i\wt{X}_j] \leq \E[\wt{Y}_i\wt{Y}_j]$.
  Using Corollary 3.12 in \cite{Ledoux1991} we know that for all $\lambda\in\bR$\,:
  \begin{equation}
    \label{eq:comparison}
    \Pr\Big[\max_{i\leq m} \wt{X}_i \geq \lambda\Big] \geq \Pr\Big[\max_{i\leq m} \wt{Y}_i \geq \lambda\Big]\,.
  \end{equation}
  Now it is easy to check that
  $\Pr\big[\max_{i\leq m}\wt{Y}_i < \lambda-\sigma\big] \leq \Pr\big[\max_{i\leq m}Y_i<\lambda\big]$
  and similarly for $\wt{X}$ that
  $\Pr\big[\max_{i\leq m}X_i < \lambda-(\sigma^2+\E Y_i^2 - \E X_i^2)^{\frac 1 2}\big] \leq \Pr\big[\max_{i\leq m}\wt{X}_i<\lambda\big]$.
  With Eq.~\ref{eq:comparison} we have:
  \[
    \Pr\Big[\max_{i\leq m} X_i<\lambda-\sigma-(\sigma^2+\E Y_i^2 - \E X_i^2)^{\frac 1 2}\Big]
    \leq \Pr\Big[\max_{i\leq m}Y_i<\lambda\Big]\,.
   \]
   Using that $\E X_i^2 \geq \E Y_i^2$ finishes the proof.
\end{proof}

\subsection{Proof of the Lower Bound}
We now use the previous lemmas to bound from below $\sup_{x\succ s}f(x)-f(s)$
for a node $s$ satisfying properties of a pruned node.
By doing so, we give the exact formula for the function $\varphi$ in Eq.~\ref{eq:varphi}.

\begin{lemma}[Anti-concentration for a pruned node]
  \label{lem:max_one_lvl}
  Let $s\in\cT_h$ and $(t_i)_{i\leq m}$ such that $t_1=s$
  and for all $2\leq i\leq m$, $p(t_i)=s$ and $d(s,t_i)\leq \Delta$.
  If $d(t_i,t_j) \geq \alpha$ for all $i\neq j$ then
  the following holds with probability at least $1-e^{-u}$ for $3u<m$\,:
  \[\max_{i\leq m} f(t_i)-f(s) \geq \frac{\alpha}{\sqrt{2}} \sqrt{\log \frac{m}{3u}} - 2\Delta\,.\]
\end{lemma}
\begin{proof}
  For $i\leq m$, let $X_i=f(t_i)-f(s)$ and $Y_i \iid \cN(0,\frac{\alpha^2}{2})$ be
  independent Gaussian variables.
  We have $\E(X_i-X_j)^2=d(t_i,t_j)^2\geq \alpha^2 = \E(Y_i-Y_j)^2$
  and $\Delta^2\geq \E X_i^2 \geq \alpha^2 > \E Y_i^2$ since $X_1=0$.
  Then using Lemma~\ref{lem:comparison} we know that for all $\lambda\in\bR$\,:
  \[\Pr\Big[\max_{i\leq m} X_i<\lambda - 2\Delta\Big] \leq \Pr\Big[\max_{i\leq m}Y_i<\lambda\Big]\,.\]
  Now using Lemma~\ref{lem:max_normal} we obtain for $m \geq 3u$\,:
  \[\Pr\left[\max_{i\leq m} X_i<\frac{\alpha}{\sqrt{2}}\sqrt{\log\frac{m}{3u}} - 2\Delta\right] \leq e^{-u}\,.\]
\end{proof}

The following lemma describes the key properties of the tree $\cT$ as computed by Algorithm~\ref{alg:tree_lb}.
We show that the supremum $\sup_{x\succ s}f(x)-f(s)$ at every depth is bounded from below
by the sum of the values found in Lemma~\ref{lem:max_one_lvl}, up to constant factors.

\begin{lemma}[Anti-concentration for the tree]
  \label{lem:tree_induction}
  Fix any $u>0$ and set accordingly $u_i=u+2^i+i\log 2$ for any $i>0$.
  For $\cT$ the tree obtained by Algorithm~\ref{alg:tree_lb},
  we have for all $s\in\cT_h$ with probability at least $1-e^{-u_h}$ that:
  \[\sup_{x\succ s}f(x)-f(s) \geq c_u^{-1} \sup_{x\succ s} V_h(s,x)\,,\]
  where $V_h(s,x)=\sum_{i=h}^\infty \Delta_i(x) \Big( \sqrt{2^{i-3}-\frac 1 8 \log(3u_i+3\log 2)}-2\Big)$,
  and $\Delta_i(x)$ is the radius of the cell of $x$ at depth $i$,
  and $c_u\in\bR$ depends on $u$ only.
\end{lemma}
\begin{proof}
  We first show that we can restrict the study of $V_h(s,x)$ to only the summands obtained by pruning $\cT$,
  up to constant factors.
  To lighten the notations, let's write:
  \[b_i := \sqrt{2^{i-3}-\frac 1 8 \log(3u_i+3\log 2)}-2\,.\]
  Then for a sequence $t_h=p_h(x), \dots, t_{h+j}=p_{h+j}(x)$ of parents of $x$,
  if $t_h$ is the single pruned node, then,
  \begin{align*}
    \sum_{i=h}^{h+j-1} \Delta_i(x) b_i
      &= \Delta_h(x) \sum_{i=h}^{h+j-1} 2^{h-i} b_i \\
      &\leq c_u \Delta_h(x) b_h\,,
  \end{align*}
  where $c_u\in\bR$ depends on $u$ only,
  and we used that $\Delta_{h+i}(x)$, the radius of the cell at depth $h+i$ containing $x$,
  decreases geometrically for non-pruned nodes.
  By denoting $\cP_h(x)$ the set of parents of $x$ from depth $h$
  which are pruned nodes,
  we thus proved for all $x\in\cX$:
  \begin{equation}
    \label{eq:tree_induction_lrt}
    V_h'(s,x) := \sum_{t_i\in\cP_h(x)}\Delta_i(t_i)b_i \geq c_u^{-1} V_h(s,x)\,.
  \end{equation}
  We now prove Lemma~\ref{lem:tree_induction} by showing that 
  $\sup_{x\succ s}f(x)-f(s) \geq V_h'(s,x^\star)$ for all $x^\star\succ s$ with probability at least $1-e^{-u_h}$,
  by backward induction on $\cP_h(x)$, from the deepest nodes to the shallowest ones.
  Since for the leaves $\sup_{x\succ s} f(x)-f(s) = 0 = V_h'(s,x^\star)$,
  the property is initially true.
  Let's assume it is true at depth $h'>h$ and prove it at depth $h$.
  Let $s\in\cT_h$ and $x^\star\in\cX$.
  If $p_{h+1}(x^\star)$ is not pruned, we have nothing to do and just call the induction hypothesis
  with $\sup_{x\succ s}f(x)-f(s) \geq \sup_{x\succ t}f(x)-f(t)$ where $p(t)=s$.
  Otherwise note that,
  \begin{align}
    \notag
    \sup_{x\succ s}f(x)-f(s) &= \max_{t:p(t)=s} \Big\{ f(t)-f(s)+ \sup_{x\succcurlyeq t}f(x)-f(t)\Big\}\\
    \label{eq:tree_induction_split}
    &\geq \max_{t:p(t)=s} \Big\{f(t)-f(s)\Big\} + \min_{t:p(t)=s} \Big\{\sup_{x\succcurlyeq t}f(x)-f(t)\Big\}\,.
  \end{align}
  Since the children have been pruned, we know that their number is $e^{2^h}$.
  Now thanks to Lemma~\ref{lem:max_one_lvl}, with probability at least $1-\frac 1 2 e^{-u_h}$,
  \begin{equation}
    \label{eq:tree_induction_max}
    \max_{t:p(t)=s} f(t)-f(s) \geq \frac{\Delta_h(x^\star)}{2\sqrt{2}}\sqrt{2^h-\log(3u_h+3\log 2)}-2\Delta_h(x^\star)
    = \Delta_h(x^\star)b_h\,,
  \end{equation}
  where we used that $d(t_i,t_j) \geq \frac 1 2 \Delta_h(x^\star)$ for $p(t_i)=p(t_j)=s$ by construction of $\cT$.
  Now by the induction hypothesis and a union bound, we have with probability at least $1-e^{-u_{h+1}+2^h}$
  that:
  \begin{equation}
    \label{eq:tree_induction_ih}
    \min_{t:p(t)=s} \sup_{x\succcurlyeq t}f(x)-f(t) \geq \min_{t:p(t)=s} \sup_{x\succ t} V'_{h+1}(t,x)\,.
  \end{equation}
  By construction of the pruning procedure,
  we know that the children minimizing $\sup_{x\succ t}V'_{h+1}(t,x)$ is the pruned node $p_{h+1}(x^\star)$.
  With $u_{h+1}-2^h = u_h+\log 2$, the results of Eq.~\ref{eq:tree_induction_ih}
  holds with probability at least $1-\frac 1 2 e^{-u_h}$,
  we thus obtain with probability at least $1-e^{-u_h}$:
  \[ \sup_{x\succ s}f(x)-f(s) \geq V'_h(s,x^\star)\,,\]
  which uses Eq.~\ref{eq:tree_induction_split} together with Eq.~\ref{eq:tree_induction_max},
  closes the induction and the proof of Lemma~\ref{lem:tree_induction} with Eq.~\ref{eq:tree_induction_lrt}.
\end{proof}

The proof of Theorem~\ref{thm:lower_bound} follows from Lemma~\ref{lem:tree_induction}
by a union bound on $h\in\bN$ and remarking that $\omega_h \geq \sup_{x\succ s} V_h(s,x)$
up to constant factors.

\end{document}